\pdfoutput=1
\documentclass[runningheads]{llncs}

\usepackage[T1]{fontenc}
\usepackage[utf8]{inputenc}
\usepackage{amsmath,amssymb}
\usepackage{booktabs}
\usepackage{stmaryrd}
\usepackage{graphicx}
\usepackage{tikz}
\usepackage{xcolor}
\usepackage{hyperref}
\hypersetup{
  colorlinks=true,
  linkcolor=blue,
  citecolor=blue,
  urlcolor=blue,
  pdftitle={Decision-Aware Approximation of Belief Functions for Evidential Combinatorial Optimization},
  pdfauthor={Sohaib Afifi},
  pdfkeywords={belief functions, decision-aware approximation, combinatorial optimization,
    decision regret, evidential shortest path}
}
\newcommand{\Crit}{\mathrm{Crit}}
\newcommand{\xstar}{x^\star}
\newcommand{\mhat}{\hat m}
\newcommand{\cbar}{\bar c}
\newcommand{\lo}{\underline a}
\newcommand{\hi}{\overline a}
\newcommand{\R}{\mathbb{R}}
\newcommand{\preceqd}{\preceq_d}
\newcommand{\rev}[1]{#1}
\newcommand{\trim}[1]{#1}
\newcommand{\rr}[1]{#1}

\begin{document}

\title{Decision-Aware Approximation of Belief Functions for Evidential Combinatorial Optimization}
\titlerunning{Decision-Aware Approximation of Belief Functions}

\author{Sohaib Afifi\inst{1}}
\authorrunning{S. Afifi}
\institute{Univ.\ Artois, UR 3926, Laboratoire de G\'enie Informatique et
d'Automatique de l'Artois (LGI2A), B\'ethune, France\\
\email{sohaib.afifi@univ-artois.fr}}

\maketitle

\begin{abstract}
Reducing the number of focal elements of a mass function is classically driven by an
intrinsic distance, such as Jaccard or Jousselme, that keeps the approximation close to the
original as a body of evidence. We consider instead the case where the mass function feeds a linear
combinatorial optimisation problem with evidential costs. What should then be preserved is
not the closeness of the two mass functions, but the quality of the decision they induce. We
introduce a decision-aware approximation that targets the regret of the decision: one
decides with the cheaper approximation and is evaluated under the true mass function. On a
minimal shortest path, the distance-optimal approximation flips the decision while a
decision-aware merge preserves it, and this occurs on a non-negligible fraction of random
instances. We prove a
one-point bound that localises the regret at the true optimum, turn it into an exact dynamic
program for the scalar case, and extend it to an online version that prunes focal elements
before the final cost is known. In experiments the decision-aware compressor flips the decision
less often than representation-aware compression, for both the linear criterion and a non-linear
proxy read-out.
\keywords{Belief functions \and Decision-aware approximation \and Combinatorial
optimization \and Decision regret \and Evidential shortest path.}
\end{abstract}

\section{Introduction}
\label{sec:intro}

Mass functions built by fusion or by propagation often carry many focal elements. Reducing
their number is a common need. Intrinsic distances are the natural tool when the goal is to
\emph{summarise the evidence}: they keep the approximation close to the original as a body of
evidence.

Our setting is different. The belief function is an \emph{input} to a combinatorial
optimisation, and what we care about is the \emph{decision} it leads to. We propose a
\emph{decision-aware} approximation for this case. Instead of keeping the approximation close
to the original, it \rev{aims to keep} the decision taken under the approximation optimal for the true
problem. \trim{Intrinsic approximations solve a different problem: when the mass function is only
an intermediate used to make a decision, preserving the decision can matter more than preserving
it. In transferable-belief-model terms~\cite{smets1994tbm}, $m$ is the credal object and $\mhat$ a
computational shortcut to the decision, not a revised belief.} The two goals are
complementary and need not pick the same approximation.

\paragraph{Contributions.}
\begin{enumerate}
  \item We define the \emph{decision-aware} approximation and the \emph{approximation
        regret} it controls, and give a small evidential shortest path where the
        distance-optimal merge and the decision-aware merge differ, the former changing the
        decision (Sec.~\ref{sec:method}).
  \item We prove a one-point bound $R(\mhat)\le\Delta(\xstar(m))$ for any approximation that
        is monotone for the criterion (Thm.~\ref{thm:bound}).
  \item We give an exact $O(N^2K)$ decision-aware dynamic program for the scalar case
        (Prop.~\ref{prop:dp}) and an online version, and we measure both against
        representation-aware baselines on random instances (Sec.~\ref{sec:exp}).
\end{enumerate}

\section{Background and Related Work}
\label{sec:background}

\paragraph{Problem.}
We solve a linear combinatorial problem $\xstar=\arg\min_{x\in X}c^\top x$ with
$X\subseteq\{0,1\}^n$, for example the set of $s$ to $t$ paths in a graph. The cost $c$ is
uncertain. It is described by a mass function $m$~\cite{shafer1976,smets1994tbm} whose focal
elements are \emph{boxes} $F=[\lo_F,\hi_F]$, that is, one cost interval per
coordinate. \rev{We take the cost frame to be a finite integer grid, so the set
cardinalities used by the Jaccard index in Sec.~\ref{sec:method} count integer points; a
continuous-volume variant is identical with volumes in place of counts.} The criterion is linear in $m$. Since $X\subseteq\{0,1\}^n$ we have $x\ge 0$, and
the \emph{lower} expected cost is
\[
  \Crit(x;m)=\sum_F m(F)\,\min_{c\in F}c^\top x=\sum_F m(F)\,\lo_F^\top x=\cbar^\top x,
  \qquad \cbar=\sum_F m(F)\,\lo_F .
\]
The \emph{upper} expected cost is the same with the upper bound $\hi_F$. The decision
is $\xstar(m)\in\arg\min_x\Crit(x;m)$, made unique by a fixed deterministic tie-breaking
rule.

\paragraph{Safe merges.}
We order boxes by componentwise domination of both endpoints: $F\preceqd G$ iff
$\lo_F\le\lo_G$ and $\hi_F\le\hi_G$. The \emph{safe merge} of $A_i,A_j$ is
their join, the smallest box dominating both,
\[
  \lo_B=\max(\lo_i,\lo_j),\qquad \hi_B=\max(\hi_i,\hi_j),
\]
with masses added~\cite{tedjini2021specificity,denoeux2001inner}. \trim{It is the join in the
cost-dominance order (kept to make the criterion monotone), not a set-inclusion outer
approximation.} Both bounds only rise, so the lower and the upper
expected cost can only increase:
\[
  \Delta(x):=\Crit(x;\mhat)-\Crit(x;m)\ge 0 \quad\text{for } x\ge 0 .
\]
The question we study is \emph{which} safe merge to make.

\paragraph{Why approximate?}
For belief functions, and more generally $2$-monotone capacities, these problems are tractable
when their deterministic version is~\cite{vu2025capacities}. So approximation is \emph{not}
needed to make a single decision tractable. It matters for another reason: the focal set can be
large and grows when beliefs are combined along a path, so reducing it saves memory and keeps
online computation feasible. We study how to reduce it without changing the decision.
\rr{The approximation is a computational shortcut for the decision at hand, not a revised body of
evidence: $\mhat$ is built around the current decision (via $\xstar(m)$ or its online proxy) and
consumed by that optimisation, not stored in place of $m$. A later, different problem is served by
recompressing from the original $m$.}

\paragraph{Related work.}
Reducing the number of focal elements has a long history. Representation-oriented methods keep
the largest masses~\cite{tessem1993approximations}, cluster focal elements into inner and outer
approximations~\cite{denoeux2001inner}, merge them under various
relations~\cite{tedjini2021specificity}, or pick the closest approximation for an intrinsic
distance such as Jaccard or Jousselme~\cite{jousselme2001distance}; none looks at a downstream
optimisation. Closest to us, Bauer pairs approximation with decision making in the
Dempster--Shafer setting \rev{and already proposes a decision-driven approximation, but
offers no decision-regret guarantee and targets a pignistic single-attribute decision rather than
a combinatorial optimum}~\cite{bauer1997approximation}. Decision-making with belief functions is
surveyed in~\cite{denoeux2019decision}, and the corresponding combinatorial problems are
tractable for belief functions~\cite{vu2022evidential,vu2025capacities}. \rev{The
decision-regret objective is shared}
with decision-focused learning~\cite{mandi2024decision,elmachtoub2022spo,berthet2020perturbed,%
pogancic2020blackbox}, but the goal differs: those methods \emph{estimate} an unknown cost
\rev{from features}, while we \emph{compress} a known belief function to save computation and memory.

\section{Decision-Aware Approximation}
\label{sec:method}

\begin{definition}[Approximation regret]
For a $\preceqd$-safe approximation $\mhat$ of $m$,
\[
  R(\mhat)=\Crit\big(\xstar(\mhat);m\big)-\Crit\big(\xstar(m);m\big)\ \ge 0 .
\]
\end{definition}
We decide with $\mhat$ and are judged by $m$, so $R(\mhat)=0$ when the approximate
decision stays optimal for the true problem. \trim{Both approximations reduce the focal set by the
same safe merges and differ only in which merge they choose: a representation-aware one stays
closest to $m$ (grouping the most similar or keeping the most significant focals), a decision-aware
one keeps $R$ small. The two need not agree.} Note that $R$ is not a minimax
regret \emph{criterion}: it is the regret caused by approximating, and sits on top of any criterion.

\paragraph{A minimal example.}
\trim{The diamond of Fig.~\ref{fig:diamond} has two $s$ to $t$ paths and three focal boxes; we use
the lower expected cost.}

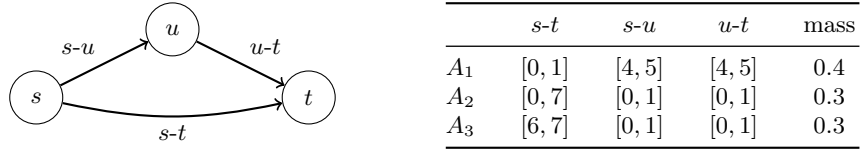
\begin{figure}[t]
\centering
\begin{minipage}[c]{0.40\textwidth}
\centering
\begin{tikzpicture}[scale=0.9]
  \node[circle,draw,minimum size=7mm] (s) at (0,0) {$s$};
  \node[circle,draw,minimum size=7mm] (u) at (2,1) {$u$};
  \node[circle,draw,minimum size=7mm] (t) at (4,0) {$t$};
  \draw[->,thick] (s) to[bend right=12] node[below]{\footnotesize $s$-$t$} (t);
  \draw[->,thick] (s) -- node[above left]{\footnotesize $s$-$u$} (u);
  \draw[->,thick] (u) -- node[above right]{\footnotesize $u$-$t$} (t);
\end{tikzpicture}
\end{minipage}\hfill
\begin{minipage}[c]{0.56\textwidth}
\centering
\begin{tabular*}{0.8\textwidth}{@{\extracolsep{\fill}}lcccc}
\toprule
 & $s$-$t$ & $s$-$u$ & $u$-$t$ & mass \\
\midrule
$A_1$ & $[0,1]$ & $[4,5]$ & $[4,5]$ & $0.4$ \\
$A_2$ & $[0,7]$ & $[0,1]$ & $[0,1]$ & $0.3$ \\
$A_3$ & $[6,7]$ & $[0,1]$ & $[0,1]$ & $0.3$ \\
\bottomrule
\end{tabular*}
\end{minipage}
\caption{The diamond instance (left): path $x_1$ is edge $s$-$t$, path $x_2$ is edges $s$-$u$
and $u$-$t$. Its three focal boxes (right), one cost interval per edge, with their masses.}
\label{fig:diamond}
\end{figure}

\trim{Here $\cbar=\sum_F m(F)\,\lo_F=(1.8,\,1.6,\,1.6)$, so $x_1$ costs $1.8$ and $x_2$ costs
$1.6+1.6=3.2$: the true optimum is $x_1$, by a margin of $1.4$.}

\trim{We compress the three boxes into two; there are three safe merges (Table~\ref{tab:merges}).
Merging $A_2,A_3$ moves mass onto the value $6$ on edge $s$-$t$, so its cost jumps to $3.6$ and
$x_2$ wins: the decision changes, $R=\Crit(x_2;m)-\Crit(x_1;m)=3.2-1.8=1.4$. The other two merges
keep $x_1$ optimal, so $R=0$.}

\begin{table}[tb]
\centering
\caption{The three safe merges from 3 to 2 focal boxes: aggregated cost $\cbar$ and path costs
(lower expected cost), Jaccard similarity \rev{(between the two merged boxes)} and Jousselme distance \rev{(between $m$ and $\mhat$)}, and the regret $R$.
Only $A_2A_3$ changes the decision, and it is also the most similar merge.}
\label{tab:merges}
\setlength{\tabcolsep}{4pt}
\begin{tabular}{lcccccc}
\toprule
merge & $\cbar$ after merge & cost $x_1$ & cost $x_2$ & Jaccard sim. & Jousselme $d$ & $R$ \\
\midrule
$A_1A_2$ & $(1.8,\,2.8,\,2.8)$ & $1.8$ & $5.6$ & $0.00$ & $0.55$ & $0$ \\
$A_1A_3$ & $(4.2,\,2.8,\,2.8)$ & $4.2$ & $5.6$ & $0.00$ & $0.61$ & $0$ \\
$A_2A_3$ & $(3.6,\,1.6,\,1.6)$ & $3.6$ & $3.2$ & $\mathbf{0.25}$ & $\mathbf{0.26}$ & $\mathbf{1.4}$ \\
\bottomrule
\end{tabular}
\end{table}

\trim{Both intrinsic measures in Table~\ref{tab:merges} single out $A_2A_3$, the merge that changes
the decision, as the closest to $m$. This is no flaw of measures that serve a different goal, but a
decision-aware choice keeps $R=0$ instead.}

\paragraph{A one-point bound.}
The regret is controlled by the distortion at the \emph{single} true optimum. The result needs
only that the approximation does not lower the criterion.

\begin{theorem}[One-point bound]
\label{thm:bound}
Let $\mhat$ be an approximation that is monotone for the criterion, that is
$\Delta(x):=\Crit(x;\mhat)-\Crit(x;m)\ge 0$ for all $x\in X$.
Then
\[
  R(\mhat)\ \le\ \Delta\big(\xstar(m)\big).
\]
\end{theorem}

\begin{proof}
Write $\hat x=\xstar(\mhat)$ and $\xstar=\xstar(m)$. Since $\hat x$ minimises
$\Crit(\cdot;\mhat)$, $\Crit(\hat x;\mhat)\le\Crit(\xstar;\mhat)$. Using $\Delta(\hat x)\ge 0$ and this optimality,
\[
  \begin{aligned}
  R(\mhat)=\Crit(\hat x;m)-\Crit(\xstar;m)
    &\le \Crit(\hat x;\mhat)-\Crit(\xstar;m)\\
    &\le \Crit(\xstar;\mhat)-\Crit(\xstar;m)=\Delta(\xstar).
  \end{aligned}
\]
Finally $R\ge 0$ because $\hat x\in X$ and $\xstar$ minimises $\Crit(\cdot;m)$ over $X$. \qed
\end{proof}

\rev{The chain invokes $\Delta\ge0$ only at $\hat x$ and $\xstar$, so the computable bound is
evaluated only at the true optimum, and it holds for \emph{any} minimiser, hence is insensitive to
how ties in $\arg\min$ are broken.}

\rr{Monotonicity is the whole hypothesis. Here it holds because the expected costs are linear in
$m$ and the safe merge raises every bound, so $\Delta(x)=(\cbar(\mhat)-\cbar(m))^\top x\ge 0$ for
$x\ge 0$; more generally it holds for any criterion isotone in the bounds the merge raises, covering
both the lower and the upper expected cost. It is not automatic: for a minimax-regret criterion,
monotonicity under $\preceqd$ is open, and there Theorem~\ref{thm:bound} applies only
conditionally.}
The regret is
thus \rev{bounded by} the gap at the true optimum, \trim{non-zero only where the merged lower
bounds differ, so the bound is local and computable without re-solving.}

\rr{\paragraph{Not a probability collapse.}
Reading only $\lo_F$ makes the \emph{criterion} a linear functional of $m$, but it does not turn
$\mhat$ into a probability. The object stays set-valued: the safe merge is the join in the
cost-dominance order and moves \emph{both} endpoints, $\lo_B=\max(\lo_i,\lo_j)$ and
$\hi_B=\max(\hi_i,\hi_j)$, so which merges are admissible depends on the boxes, not on their lower
corners alone. The same $\mhat$ then also serves the upper expected cost, the robust reading we
report alongside, which a singleton or pignistic transform could not. Criteria using the whole
interval, such as minimax regret or a Hurwicz mix, are where the box structure is indispensable, and
the decision-aware bound still targets them through Theorem~\ref{thm:bound} under the same open
monotonicity condition.}

\paragraph{\rev{From bound to algorithms.}}
\rev{The theorem turns regret minimisation, which would re-solve the problem for every candidate
compression, into a one-solve surrogate: solve once for $\xstar(m)$, then pick the merge that least
distorts that decision. This is an offline objective (Bound-DA below); a streaming setting must
instead replace the not-yet-known $\xstar(m)$ by a local sensitivity $w_t$ (Online-DA). The theorem
yields the offline surrogate, not an online compressor.}

\paragraph{An exact dynamic program.}
We seek the safe merge that makes the bound small. With weight $w=\xstar(m)\ge0$ this means
minimising $w^\top(\cbar(\mhat)-\cbar(m))$ over the groupings of the $N$ focal boxes into $K$
groups; in the scalar case this is ordered clustering on the line.

\begin{proposition}[Exact scalar dynamic program]
\label{prop:dp}
For a scalar frame ($\cbar\in\R$) the grouping of the $N$ focal boxes into $K$ groups that
minimises $w^\top(\cbar(\mhat)-\cbar(m))$, with $w=\xstar(m)\ge0$ fixed (a positive scalar that
normalises away in the scalar frame), is found exactly in $O(N^2K)$.
\end{proposition}

\begin{proof}
Sort the boxes so that the relevant bound satisfies $a_1\le\dots\le a_N$ (write $a_\ell$ for the
relevant bound: the lower bound for the lower expected cost, the upper bound for the upper). Merging a block into
one box replaces its bound by the block maximum, the last sorted element. The increase of
$\cbar$ contributed by a block $[i..j]$ is therefore
$a_j\,W(i,j)-\sum_{\ell=i}^{j}m_\ell a_\ell$, where $W(i,j)=\sum_{\ell=i}^{j}m_\ell$. Summed
over any partition, the subtracted term equals $\sum_\ell m_\ell a_\ell=\cbar(m)$, a constant.
Minimising the total increase is thus minimising $\sum_{\text{groups}}a_{\max}\,W$, where
$a_{\max}$ is the largest bound in the group and $W$ its mass.

An optimal grouping can be taken with contiguous blocks. Let the group maxima of an optimal
grouping be $a_{j_1}\le\dots\le a_{j_K}=a_N$ (ties broken arbitrarily), and reassign every box to the group whose maximum is
the smallest one not below the box's own bound. The box attaining a group maximum $a_{j_r}$ maps
back to its own group ($a_{j_r}$ is the smallest group maximum not below itself), so every group
keeps its maximum and stays non-empty: exactly $K$ groups remain. Each box then pays the smallest available
maximum, so the objective does not increase, and the groups become the intervals of sorted
boxes between consecutive maxima. With contiguous blocks $[i..j]$ of cost
$\mathrm{cost}(i,j)=a_j\,W(i,j)$, the recurrence
\[
  D[j,K]=\min_{i\le j}\big\{\,D[i-1,K-1]+\mathrm{cost}(i,j)\,\big\},\qquad D[0,0]=0,
\]
gives the optimum in $O(N^2K)$. \qed
\end{proof}

The block cost satisfies the quadrangle inequality, the relevant difference being
$(a_{j'}-a_j)\,W(i,i'-1)\ge 0$; it is therefore Monge, which suggests standard acceleration for
ordered clustering, but we keep the $O(N^2K)$ version ($N$ counts focal boxes, not graph size).
\trim{For general vector boxes the grouping is no longer one-dimensional; the scalar program is
then exact for scalar costs and, otherwise, a heuristic that sorts boxes by $w^\top(\text{bound})$
(exact when the bounds form a chain).}

\paragraph{Online compression.}
The need is strongest when focal elements pile up: along a path the cost so far is built by
combining the edge beliefs, and the focal count grows fast, so we must compress before the final
cost, hence before the true optimum, is known. At step $t$, with $m_t$ the accumulated belief
and $\mhat_t$ its compression to $K$ boxes, the compressor picks the $K$-grouping minimising
$w_t\,(\cbar(\mhat_t)-\cbar(m_t))$, \rev{where $w_t\ge 0$ is a local sensitivity of the decision. In our online experiments the accumulated cost is one-dimensional, so $w_t$ normalises away and Online-DA is the \emph{scalar specialisation of the bound objective}: it minimises the per-step drift of the expected cost, not a sensitivity distinguishing branches or downstream alternatives};
this is the per-step one-point bound. For the linear read-out the distortion is additive,
$\Delta_{\mathrm{tot}}(x)=\sum_t\Delta_t(x)$, so the regret is bounded by $\sum_t\Delta_t(\xstar(m))$;
\trim{Online-DA replaces the unknown $\xstar(m)$ by $w_t$, a deployable proxy, not a certified
minimiser, and non-linear read-outs break additivity.}
\rr{The break is precise: under a non-linear read-out $g$ of the accumulated cost the per-step
bounds no longer telescope, so a small $\Delta_t$ no longer controls the regret and no per-step
guarantee survives. Online-DA optimises the linear surrogate at every step regardless, hence a
proxy under the non-linear read-out.}

\section{Experiments}
\label{sec:exp}

\paragraph{Instances.}
\trim{Each instance is a small shortest path with $3$ to $5$ edges split into $2$ to $4$ disjoint
branches. In the static study each focal box has, per edge, an integer lower bound in
$\{0,\dots,7\}$ and width in $\{0,1,2\}$, masses uniform on the simplex, and $3$ to $6$ focal
boxes. In the online study each edge carries $2$ or $3$ scalar focal intervals (lower bound in
$\{0,\dots,19\}$, width in $\{0,\dots,4\}$), branches have $2$ to $4$ edges, and the non-linear
budget is the mean expected cost over branches. Ties in $\arg\min$ break by index; we use $5$
seeds of $500$ instances ($2500$ per study) with $95\%$ confidence intervals.}

\paragraph{Methods.}
Representation-aware baselines: \emph{Jousselme}, closest in Jousselme distance (exact in the
static study, greedy in the online study where the brute force is intractable); \emph{Jaccard},
a greedy merge of the closest pair; \emph{largest-mass}, the $K-1$ heaviest boxes, after
Tessem~\cite{tessem1993approximations}; and a \emph{random} safe merge. Decision-aware variants:
\emph{Oracle-DA} minimises the exact regret (solving the true problem for every grouping, a
non-deployable gold standard); \emph{Bound-DA} targets the one-point bound
$\Delta(\xstar(m))$ via one true solve plus the program of Sec.~\ref{sec:method} (exact in the
scalar case, a projection heuristic for the vector instances here); \emph{Online-DA} is the
streaming proxy. \trim{The static study uses Bound-DA against Oracle-DA, the online study Online-DA.}

\rr{\paragraph{Choice of $K$.}
$K$ is a resource budget set by the available memory or latency, not tuned to the data; the method
compresses to whatever $K$ is given. The tables trace the trade-off: the decision-change rate falls
as $K$ grows, and the decision-aware advantage over the baselines is widest at the smallest $K$,
where compression bites hardest, so decision-awareness matters most when the budget is tight.}

\paragraph{Static study.}
\trim{Table~\ref{tab:freq} reports the lower expected cost at $K=2$. Bound-DA nearly matches the
Oracle-DA lower bound, changes the decision about six times less often than the representation-aware
methods, and has smaller mean and tail regret, at the price of less faithfulness (larger
$d_{\mathrm{Jou}}$). At $K=3$ the ordering is the same (Bound-DA $0.7\%$ against $9$ to $14\%$), and
the upper expected cost behaves identically ($2.4\%$ at $K=2$).} \rev{This isolates the objective,
not computational necessity (a single decision is already polynomial): intrinsic distance and
decision regret genuinely differ, and Bound-DA (the projection heuristic here, not the exact scalar
DP) nearly matches Oracle-DA.}

\begin{table}[tb]
\centering
\caption{Static study, lower expected cost, $K=2$ ($2500$ instances, $5$ seeds). Fraction with
$R>0$ ($\pm95\%$ CI), mean, $95$th-percentile and maximum regret, and mean Jousselme distance
$d_{\mathrm{Jou}}(m,\mhat)$. Oracle-DA is a non-deployable lower bound on the achievable decision-change
rate.}
\label{tab:freq}
\begin{tabular*}{\textwidth}{@{\extracolsep{\fill}}lccccc@{}}
\toprule
method & dec.-chg.\% & mean $R$ & $q_{95}(R)$ & max $R$ & $d_{\mathrm{Jou}}$ \\
\midrule
Jousselme         & $13.6\pm1.3$ & $0.085$ & $0.67$ & $4.81$ & $0.45$ \\
Jaccard           & $16.9\pm1.5$ & $0.128$ & $0.94$ & $3.82$ & $0.58$ \\
largest-mass      & $14.2\pm1.4$ & $0.094$ & $0.71$ & $4.81$ & $0.48$ \\
random safe       & $18.0\pm1.5$ & $0.143$ & $1.04$ & $3.89$ & $0.59$ \\
\textbf{Bound-DA} & $\mathbf{2.5\pm0.6}$ & $\mathbf{0.010}$ & $\mathbf{0.00}$ & $2.21$ & $0.56$ \\
Oracle-DA         & $1.8\pm0.5$ & $0.006$ & $0.00$ & $2.21$ & $0.60$ \\
\bottomrule
\end{tabular*}
\end{table}

\paragraph{Online study.}
\trim{Each branch builds its cost step by step, accumulating about $40$ focal boxes before
compression. We compress to $K$ after each step, decide with the compressed belief, and judge with
the full belief, under a linear read-out (expected cost) and a non-linear one (plausibility that
the cost passes a budget).} \trim{Exact Jousselme grouping is intractable here, so we use a greedy
variant, which is itself the point of a polynomial rule.} Table~\ref{tab:online} reports the
fraction of decision changes. Under the linear read-out, the one covered by the theory, Online-DA
changes the decision least at every $K$, below even the greedy-Jousselme baseline, and its
worst-case regret is the smallest too. Under the non-linear read-out, where Online-DA is only a
proxy, it is on par with greedy-Jousselme and still well below the mass-based and random baselines.
\rev{Here compression is genuinely necessary, the full belief and true optimum not yet being
available; the non-linear read-out is a stress test outside the theory.}

\begin{table}[tb]
\centering
\caption{Online study ($2500$ instances, $5$ seeds). Fraction of decision changes ($R>0$, in
percent); $95\%$ CIs are at most $\pm1.9$. Online-DA is clearly lowest under the linear
read-out; under the non-linear read-out, where it is only a proxy, it ties the greedy-Jousselme
baseline.}
\label{tab:online}
\begin{tabular*}{\textwidth}{@{\extracolsep{\fill}}llccccc@{}}
\toprule
read-out & $K$ & \textbf{Online-DA} & Jous.\ greedy & Jaccard & l.-mass & random \\
\midrule
linear     & 2 & $\mathbf{11.1}$ & $13.8$ & $16.9$ & $17.3$ & $18.6$ \\
linear     & 3 & $\phantom{0}\mathbf{6.0}$ & $\phantom{0}9.4$ & $12.2$ & $14.8$ & $17.4$ \\
linear     & 4 & $\phantom{0}\mathbf{4.5}$ & $\phantom{0}7.0$ & $\phantom{0}8.6$ & $13.0$ & $16.0$ \\
non-linear & 2 & $\mathbf{19.6}$ & $21.1$ & $24.3$ & $26.5$ & $34.1$ \\
non-linear & 3 & $\mathbf{10.5}$ & $10.8$ & $14.9$ & $16.8$ & $24.7$ \\
non-linear & 4 & $\phantom{0}7.0$ & $\phantom{0}\mathbf{6.5}$ & $\phantom{0}9.3$ & $12.0$ & $19.0$ \\
\bottomrule
\end{tabular*}
\end{table}

\section{Conclusion}
\label{sec:conclusion}

We proposed a decision-aware way to approximate a belief function that feeds a combinatorial
decision: control the decision regret rather than an intrinsic distance, via a one-point bound, an
exact scalar dynamic program, and an online version. In experiments a deployable bound-aware rule
nearly matches a clairvoyant oracle and changes the decision less often than representation-aware
baselines, at the cost of a larger representation distance. \trim{Two points stay open: exact
grouping in the vector case is not one-dimensional clustering (our scalar program is then a
projection heuristic), and the monotonicity behind Theorem~\ref{thm:bound} is only conjectured for
a minimax regret criterion. Only the single lower bound $\lo_F$ enters the lower expected cost,
which reduces it to one vector $\cbar$ and enables the program, while the regret and the bound
hold for arbitrary focal sets.} \rr{Although the experiments use shortest paths, nothing in the
construction is specific to them: Theorem~\ref{thm:bound} needs only a criterion monotone under the
safe merge over $x\in X\subseteq\{0,1\}^n$, and the scalar program depends on the focal count, not
on $X$. Any linear evidential $0$--$1$ problem, such as knapsack or assignment, fits the same
template with its own deterministic solver.} \trim{Extending the approach to uncertain constraints
is left for future work.}

\bibliographystyle{splncs04}
\bibliography{refs}

\end{document}